%% file: Tehranietal12MSarchive.tex
\documentclass[12pt,onecolumn,journal,draftclsnofoot,doublespace]{IEEEtran}
\usepackage{epsf,psfrag,amssymb,amsfonts,amsmath,graphicx,cite,subfigure}

\input macros

\newcounter{algleo}
\newlength{\lefttab}
  {\trivlist
   \topsep=0pt\itemsep=0pt
   \addtolength{\lefttab}{1.25em}
   \leftskip=\lefttab}%
  {\endtrivlist}

\def\mbbE{\mathbb{E}}
\def\nn{{\nonumber}}
\def\scalefig#1{\epsfxsize #1\textwidth}

\newtheorem{theorem}{Theorem}

\begin{document}

\title{\LARGE\bf   Dynamic Pricing under Finite Space Demand Uncertainty: A Multi-Armed Bandit with Dependent Arms}
\author{Pouya Tehrani, Yixuan Zhai, Qing Zhao\\
 Department of Electrical and Computer Engineering\\
 University of California, Davis, CA 95616,
 \{potehrani,yxzhai,qzhao\}@ucdavis.edu\\[.5em]
  \bf }
\maketitle%

\begin{abstract}
%\vspace{-0.8em}

We consider a\footnotetext{This work was supported in part by the
National Science Foundation under Grant CCF-0830685 and by the Army
Research Office under Grant W911NF-08-1-0467.} dynamic pricing
problem under unknown demand models. In this problem a seller offers
prices to a stream of customers and observes either success or
failure in each sale attempt. The underlying demand model is unknown
to the seller and can take one of $N$ possible forms. In this paper,
we show that this problem can be formulated as a multi-armed bandit
with dependent arms. We propose a dynamic pricing policy based on
the likelihood ratio test. We show that the proposed policy achieves
complete learning, \ie it offers a bounded regret where regret is
defined as the revenue loss with respect to the case with a known
demand model. This is in sharp contrast with the logarithmic growing
regret in multi-armed bandit with independent arms.

\end{abstract}

\vspace{1em}


\section{Introduction}\label{sec:inc}
The sequential pricing of a certain good under an unknown demand
model is a fundamental management science problem and has various
applications in financial services, electricity market, online
posted-price auctions of digital goods, and radio spectrum
management. In this problem, a seller offers a sequence of prices of
the good to a stream of potential customers and observes either
success or failure in each sale attempt. The characteristic of each
customer is assumed to be identical and is described by a demand
model $\rho(p)$ which prescribes the probability of a successful
sale at the offered price $p$. The demand model is assumed to be
unknown to the seller and needs to be learnt online through
sequential observations. Unlike the conventional operation research
and management science constraint on the inventory, we assume that
there is an unlimited supply of the good (consider, for example, the
online posted-price auction where there is an infinity supply of the
digital good).
 The objective is to maximize the total revenue over a horizon of length $T$ by choosing sequentially
the price at each time based on the sale history. When choosing the
price at each step, the seller confronts a tradeoff between
exploring the demand model (learning) and exploiting the price with
the best selling history (earning). As the seller gains information
about the unknown demand model from the past selling history, the
seller's pricing strategy can improve over time.
%A fundamental
%problem to study is the cost of not knowing the demand model. We aim
%to compare the expected revenue of the seller's strategy with the
%expected revenue of the optimal fixed-price strategy under the known
%demand model. This metric is called regret or cost of learning which
%is introduced in the multi-armed bandit literature.

% As a conventional assumption
%(\cite{??5,7,8,13??}) we consider auctions that are
%\emph{strategyproof} meaning that the customers have no incentive to
%lie about their demand model. It is shown in \cite{??5??} that this
%assumption is equivalent to requiring the seller to offer a price to
%a customer only based on the observation from previous customers.

\subsection{Dynamic Pricing as A Multi-Armed Bandit}

Dynamic pricing can be formulated as a special multi-armed bandit
(MAB) problem, and the connection was explored as early as 1974 by
Rothschild in~\cite{Rothschild:74}. A mathematical abstraction of
MAB in its basic form involves $N$ independent arms and a single
player. Each arm, when played, offers independent and identically
distributed (i.i.d.) random reward drawn from a distribution with
unknown mean $\theta_i$. At each time, a player chooses one arm to
play, aiming to maximize the expected total rewards obtained over a
horizon of length $T$. Depending on whether the unknown mean
$\theta_i$ of each arm is treated as random variables with known
prior distributions or as a deterministic quantity, MAB problems can
be formulated and studied within either a Bayesian or a non-Bayesian
framework.

Within the Bayesian framework, system unknown parameters are random
variables, and the design objective is policies with good
\emph{average} performance (averaged over the prior distributions of
the unknowns). Often, the performance of a policy is measured by the
total discounted reward or the average reward over an infinite
horizon. By treating the player's \emph{a posteriori} probabilistic
knowledge (updated from the \emph{a priori} distribution using past
observations) on the unknown parameters as the system state, Bellman
in 1956 abstracted and generalized the Bayesian MAB to a special
class of Markov decision processes (MDP)~\cite{Bellman:56}. In
1970s, Gittins showed that the optimal policy of MAB has a simple
index structure---the so-called Gittins index policy~\cite{Gittins}.
This leads to linear (in the number $N$ of arms) complexity in
finding the optimal policy, in contrast to the exponential
complexity one would have to face if the problem was solved as a
general MDP.

Within the non-Bayesian framework, the unknowns in the reward models
are treated as deterministic quantities and the design objective is
universally (over all possible values of the unknowns) good
policies. A commonly used performance measure is the so-called
regret (a.k.a. the cost of learning) defined as the expected total
reward loss with respect to the ideal scenario of known reward
models (under which the arm with the largest reward mean is always
played). To minimize the regret, the player needs to identify the
best arm without engaging other inferior arms too often. In 1985,
Lai and Robbins \cite{Lai&Robbins:85} showed that the minimum regret
grows at logarithmic order with $T$ and constructed a policy to
achieve the minimum regret for certain reward distributions.

The connection between MAB and dynamic pricing is now readily seen:
each potential price $p$ is an arm with an unknown reward mean
$p\rho(p)$ (the expected revenue at price $p$). When the seller can
choose any price within an interval, the problem becomes a
continuum-armed
bandit~\cite{Kleinberg:04,Aueretal:07,Agrawal:95,Cope:09,Kleinberg&Leighton:03}.
Kleinberg and Leighton in \cite{Kleinberg&Leighton:03} specifically
consider an online posted price auction under an unknown demand
model which is a special case of the continuum-armed bandit problem.
In~\cite{Rothschild:74}, Rothschild considered the case where the
seller can choose prices from a finite set and formulated the
problem as a classic MAB within the Bayesian framework assuming
prior probabilistic knowledge of the demand model. His focus was on
the question whether a seller who follows an optimal policy (in
terms of total discounted revenue over an infinite horizon) will
eventually obtain complete information about the underlying demand
model thus settle at the optimal price. It was shown
in~\cite{Rothschild:74} that the answer is in general negative. In
light of the theories on MAB developed since 1974, this conclusion
is, perhaps, no longer surprising. The optimal policy of a Bayesian
MAB will always settle at a single arm (after a finite number of
visits to other arms) which is not necessarily the best one.
Following Rothschild, McLennan showed that incomplete learning can
occur even when the seller can choose among a continuum of
prices~\cite{McLennan:84}. McLennan adopted a simple binary demand
model: it is known that one of two possible demand models
$\rho_0(p)$ and $\rho_1(p)$ pertains with prior probability $1-q_0$
and $q_0$, respectively.

Even though the optimal policy offers the best performance
\emph{averaged} over all possible demand environments under the
known prior distribution, the fact that incomplete learning occurs
with positive probability can be unsettling given that a seller may
only see one realization of the demand model and thus cares about
only the revenue under this specific realization rather than on the
average over all realizations it \emph{might} have seen. In this
case, a policy that grantees complete learning under every possible
demand model may be desirable, even though it may not offer the best
average performance.

This issue was addressed by Harrison \etal in~\cite{BDPP:10} where
they adopted the same binary demand model considered by
McLennan~\cite{McLennan:84} but focused on achieving complete
learning under each realization of the demand model rather than the
best average performance. The myopic Bayesian policy (MBP) and its
modified versions were studied. It was shown that although MBP can
lead to incomplete learning, a modified version of MBP will always
learn the underlying demand model completely and settle at the
optimal price. If we borrow the performance measure of regret that
is often used within the non-Bayesian framework, complete learning
implies a finite regret that does not grow unboundedly with the
horizon length $T$.

\subsection{Main Results}

In this paper, we provide a different approach to the problem
considered by McLennan in~\cite{McLennan:84} and Harrison \etal
in~\cite{BDPP:10}. In particular, we adopt the non-Bayesian
framework which does not assume any probabilistic prior knowledge on
which demand model may pertain. We show that completely learning
(\ie finite regret) can be achieved without this prior knowledge.
Furthermore, in contrast to the modified MBP proposed
in~\cite{BDPP:10}, our proposed policy achieves finite regret
without complete knowledge of the demand curves $\{\rho_0(p),
\rho_1(p)\}$. The only knowledge required in our proposed policy is
the optimal prices $\{p_0^*, p_1^*\}$ under each demand model and
the values $\rho_i(p_j^*)$'s ($i,j \in \{0, 1\}$) of the demand
models at these two prices. Our results also generalize to the case
with an arbitrary number $N$ of potential demand curves.

Our approach is based on a multi-armed bandit formulation of the
problem within the non-Bayesian framework. In our formulation, each
arm $1 \leq j \leq N$ represents the optimal price $p_j^*$ under
demand model $\rho_j(p)$. Since all arms share the same underlying
demand model,
 arms are correlated. In other words, observations from one arm also provide information on the quality of other
 arms by revealing the underlying demand model. Recognizing the detection component of this bandit problem
 with dependent arms, we propose a policy based on the likelihood ratio test (LRT) and show that
it has finite regret. Compared to \cite{BDPP:10}, this result on
complete learning is established in a considerably simpler manner.
Furthermore, simulation examples demonstrate that the proposed LRT
policy can outperform the modified MBP policy (CMBP) proposed in
\cite{BDPP:10}.

By introducing exploration prices (the prices with the largest
Chernoff distance between the two demand models that are currently
detected as most likely), we show that a variation of the
 LRT policy can improve the rate of learning
the underlying demand model and reduce regret. This enhancement,
however,
 requires more knowledge on the demand curves than in the LRT policy.

In the context of multi-armed bandit, this result provides an
interesting case where dependencies across arms can be exploited to
achieve finite regret that does not grow unboundedly with the
 horizon length $T$. This is in sharp contrast to a naive approach that
 ignores arm dependencies and directly applies the classic MAB policies. The latter would have led
  to a regret that grows logarithmic with $T$.

\subsection{Related Work}
Within the Bayesian framework, following Rothschild
(\cite{Rothschild:74}), Easley and Kiefer \cite{Easley&Kiefer:88}
and Aghion \etal \cite{Aghion&etal:91} also studied the
achievability of complete learning  Aviv and Pazgal in
\cite{Aviv&Pazgal:05} considered parametric uncertainty in the
demand model where a prior distribution of the unknown parameter is
assumed known. They formulated the dynamic pricing problem as a
partially observable Markov decision process (POMDP) and developed
upper bounds on the performance of the optimal policy. They also
proposed an active-learning heuristic policy with near optimal
performance. Farias and Roy in
 \cite{Farias&Roy:09} considered a similar problem under inventory constraints and
  Poisson arrivals of customers and developed near optimal heuristic policies.
 Keller and Rady in \cite{Keller&Rady:99} considered the case under infinite horizon
with discounted reward where the demand model may change over time.
In order to learn the underlying demand model, they studied two
qualitatively different heuristics based on exploration (deviating
from myopic policy) and exploitation (close to myopic policy).

 Within the non-Bayesian framework, besides regret, another metric mainly
considered in the analysis of auction mechanism
\cite{Yossef&etal:02,Blum&etal:03,Fiat&etal:02,Goldberg&etal:01} is
the competitive ratio defined as $\frac{R(S^{\text{opt}})}{R(S)}$
where $S$ is the seller's strategy, $S^{\text{opt}}$ is the optimal
fixed-price strategy under the known demand model, and $R(.)$ is the
expected revenue function. Blum \etal showed in \cite{Blum&etal:03}
that there are randomized pricing policies achieving competitive
ratio $1 + \epsilon$ for any $\epsilon > 0$. This result indicates
that $R(S)$ can converge to $R(S^{\text{opt}})$ but it does not
reveal the rate of convergence. In this paper we analyze the
additive regret ${R(S^{\text{opt}})}-{R(S)}$ (a more strict metric
than competitive ratio)
 and focus on the growth of regret with the time horizon length.
 %Segal in \cite{Segal:03}
%considers \emph{strategy-proof} (the customers have no incentive to
%lie about their demand model) off-line multi-unit auction mechanisms
%(a multi-unit auction is one in which buyers may purchase more than
%one copy of the good). Segal compares the expected regret of the
%optimal strategy-proof off-line mechanism with that of the optimal
%on-line posted-price mechanism under three assumptions on the space
%$\mathcal{D}$ of possible demand curves: $\mathcal{D}$ is a finite
%set, $\mathcal{D}$ is parameterized by a finite-dimensional
%Euclidean space, and when $\mathcal{D}$ is arbitrary.

As mentioned earlier, Kleinberg and Leighton in
\cite{Kleinberg&Leighton:03} studied dynamic pricing (online posted
price auctions) as a special continuum-armed bandit problem. In
particular, they analyzed the regret for three different cases of
demand models. In the first scenario the customer's evaluations
equal to an unknown single price in $[0, 1]$ and the customers will
only accept the offered price if it is below their evaluation.
Kleinberg and Leighton showed that there is a deterministic pricing
strategy achieving regret $O(\log \log T)$ and no pricing strategy
can achieve regret $o(\log \log T)$ where $T$ is the horizon length
(or equivalently, the number of customers). In the second scenario
the customer's evaluations are independent random samples from a
fixed unknown probability distribution on $[0, 1]$. This model
implies that the demand curve is the complement
 cumulative distribution function (CDF) of a certain random variable, which is more restrictive than a general demand model. For this scenario they
showed that there is a pricing strategy achieving regret $O(\sqrt{T
\log T})$. The last scenario considered in
\cite{Kleinberg&Leighton:03} makes no stochastic assumptions about
the demand model. It is shown that there is a pricing strategy
achieving regret $O((T^{2/3}(\log T)^{1/3})$ and no pricing strategy
can achieve regret $o(T^{2/3})$. Besbes  and Zeevi in
\cite{Besbes&Zeevi:09} considered a dynamic pricing problem under
both parametric and non-parametric uncertainty models. They obtained
lower bounds on the regret and developed algorithms that achieve a
regret close to the lower bound.

\section{Problem Statement}
Consider a seller who offers a particular product to customers who
come sequentially. For each customer, the seller proposes a price
$p$ from interval $[l,u]$; the customer accepts the price $p$ with
probability $\rho(p)$. We call function $\rho(.)$ the demand model.

Before the first customer arrives, nature chooses a demand model
from the set $\{\rho_i(.)\}_{i=0}^{N-1}$ as the ambient demand
model. This choice is unknown to the seller; but the seller has
knows the set of the potential demand models
$\{\rho_i(.)\}_{i=0}^{N-1}$ (as shown later, this assumption can be
relaxed in the proposed policy). If price $p_t$ is offered to the
$t$-th customer, the seller observes a binary random variable $o_t$
where $o_t=1$ (success) happens with probability $\rho(p_t)$ and
$o_t=0$ (failure) otherwise. The expected revenue at time $t$ if the
underlying demand model is $\rho_i(.)$ is
\begin{eqnarray}\nn
r_i(p_t)= p_t\rho_i(p_t).
\end{eqnarray}

The seller aims to maximize the total revenue by offering prices
sequentially. Under a horizon of length $T$, the pricing policy is
defined formally as the sequence $a=(a_1,a_2,\dots,a_T)$, where
$a_t$ is a map from past observations $\cup_{j=1}^{t-1}(p_j, o_j)$
to a choice of price in $[l,u]$. When there is no confusion, $a_t$
is also used to denote the action taken at time $t$.

The expected total revenue if the underlying demand model is
$\rho_i(.)$ can be written as
\begin{eqnarray*}
R_i^{a}(T) = \mbbE_i^{a}\{\sum_{t=1}^Tr_i(p_t)\}.
\end{eqnarray*}
 The regret defined as the expected revenue loss with respect to a seller who knows the underlying demand
model is given by
\begin{eqnarray*}
\Delta_i^{a}(T) = [Tr_i(p_i^*)-R_i^{a}(T)].
\end{eqnarray*}
%Define a Bayesian policy $\pi=(\pi_1,\pi_2,\dots,\pi_T)$, where
%$\pi_t$ maps from belief space $[0,1]$ to $[l,u]$. That is for a
%posterior belief $q_1,q_2,\dots$, we offer a sequence of prices
%$\{p_t\}$ where $p_t=\pi_t(q_{t-1})$.
It is easy to see that maximizing $R_i^{a}(T)$ is equivalent to
minimizing $\Delta_i^{a}(T)$.

\section{The Bayesian Approach}

In this section we give a brief review of the work by Harrison \etal
in~\cite{BDPP:10} developed within the Bayesian framework for the
special case of $N=2$. In the Bayesian approach, the seller is
equipped with priori knowledge of the underlying demand model: the
seller knows that the underlying demand model is $\rho_1(.)$ with
probability $q_0$. The objective of the seller is to maximize the
expected average revenue
\begin{eqnarray}\label{maxobjective}
\max_{a} R^{a}(T) = q_0 R_1^{a}(T) + (1 - q_0) R_0^{a}(T).
\end{eqnarray}
It is equivalent to minimizing the expected regret,
\begin{eqnarray}\label{minobjective}
\min_{a} \Delta^{a}(T) = q_0\Delta_1^{a}(T)+(1-q_0)\Delta_0^{a}(T)
\end{eqnarray}

For finite time horizon $T$, this problem can be formulated as a
partially observable Markov decision process (POMDP).

\textbf{State space}: $\mathcal {S}={0,1}$ represents demand model
$0$ or $1$ respectively.

\textbf{Action space}: $\mathcal{A}=[l,u]$ represents all possible
prices $p_t \in [l, u]$.

 \textbf{Observation space}:
$\mathcal{O}=\{0,1\}$, where $1$ represents success, and $0$
represents failure in sale.

\textbf{Transition probability}:
$p_{ij}^a=\Pr\{S_{t+1}=i|S_t=j,A_t=a\}$.

In our problem,
\begin{eqnarray}\nn
 p_{11}^a=p_{00}^a=1,\\\nn p_{01}^a=p_{10}^a=0.
\end{eqnarray}

\textbf{Observation model}: $h_{j,\theta}^a =
\Pr\{O_t=\theta|S_{t+1}=j,A_t=a\}$. We have
\begin{eqnarray}\nn
h_{00}^a&=&1-\rho_0(a),\\\nn h_{01}^a&=&\rho_0(a),\\\nn
 h_{10}^a&=&1-\rho_1(a),\\\nn
h_{11}^a&=&\rho_1(a).
\end{eqnarray}

\textbf{Immediate reward}: The instant reward in state $i$, when the
action $a$ is chosen and $\theta$ is observed is $r_{i,\theta}^a =
\theta a$.

\textbf{Policy}: $a = [a_1,\dots,a_T]$ where $a_t$ is a mapping from
the action and observation history
$\{\{A_1,\dots,A_{t-1}\},\{O_1,\dots, O_{t-1}\}\}$ to the action
space $\mathcal{A}$.

In POMDPs the sufficient statistics of the action and observation
history
\begin{eqnarray}\nn
H^{t-1} = \{\{A_1,\dots,A_{t-1}\},\{O_1,\dots, O_{t-1}\}\},
\end{eqnarray}
for choosing the optimal action at each time is the posterior
probability of the state at time  $t$. This probability is referred
to as \emph{belief} or \emph{information state} and is defined as
\begin{eqnarray}\label{belief}
q_t &=& \Pr \{S_t = 1 | H^{t-1}\},\\\nn Q_t &=& [1-q_t, q_t].
\end{eqnarray}

\textbf{Optimality equations}: The optimal policy at each time step
$t$ is a function of the current belief $q_t$ and is defined as
follows \cite{Smallwood:OR71}. Let $V_t(q_t)$ be the maximum total
expected reward obtained from time steps $t + 1$ to $T$.
\begin{eqnarray}
V_T(q_T) &=& \max_{a \in A} \{\sum_{i=0}^1 Q_T(i) \sum_{\theta=0}^1
h_{i,\theta}^a r_{i,\theta}^a\},\\\nn V_t(q_t) &=& \max_{a \in A} \{
\sum_{i=0}^1 Q_t(i) \sum_{\theta=0}^1 h_{i,\theta}^a r_{i,\theta}^a
+ \sum_{\theta =0}^1 \Pr\{O_t = \theta |
a_t\}V_{t+1}(\Gamma(q_t|a_t,\theta))\},
\end{eqnarray}
where $\Gamma(q_t|a_t,\theta)) = q_{t+1}$ is called the belief
update and is defined as
\begin{eqnarray}\label{beliefupdate}
q_{t+1} &=& \Pr \{S_{t+1} = 1| q_t, a_t, \theta_t\}\\\nn &=&
\frac{q_t\rho_1(a_t)^{\theta_t}(1-\rho_1(a_t))^{1-\theta_t}}{q_t\rho_1(a_t)^{\theta_t}(1-\rho_1(a_t))^{1-\theta_t}
+ (1 - q_t)\rho_0(a_t)^{\theta_t}(1-\rho_0(a_t))^{1-\theta_t}}.
\end{eqnarray}

The value function $V_0(q_0)$ is equivalent to \eqref{maxobjective},
and \eqref{minobjective} stated earlier.

The optimal policy of the above formulated POMDP offers the maximum
expected total revenue. However, finding the optimal policy to a
POMDP is P-SAPCE hard in general \cite{Papadimitriou:87}.

Harrison \etal in \cite{BDPP:10} considered the suboptimal myopic
policy and focused on whether finite regret (\ie complete learning)
can be achieved rather than minimizing the exact value of the
expected regret. The myopic Bayesian policy (MBP) at each step picks
the price that maximizes the current expected revenue
\begin{eqnarray}\nn
p_t = \arg\max_{p\in[l,u]} \{q_t r_1(p) + (1 - q_t)r_0(p)\}.
\end{eqnarray}
where $q_t$ is the belief at time $t$ defined in \eqref{belief} and
\eqref{beliefupdate}. For any pricing policy that offers prices from
the range $[l, u]$ it was shown in \cite{BDPP:10} that the belief
converges to a limit almost surely.

The limiting belief does not necessarily equal to $0$ or $1$
(complete learning); it is possible that the policy gets stuck at a
so-called uninformative price. The uninformative price is the price
at which both demand models $\rho_0(p)$ and $\rho_1(p)$ are equal.
In order to deal with this issue, $\delta$-discriminative policies
were considered. In particular a policy is $\delta$-discriminative
if
\begin{eqnarray}\nn |\rho_0(a_t(q)) - \rho_1(a_t(q))|
> \delta,~~~\forall t.
\end{eqnarray}
It was shown in~\cite{BDPP:10} that if a policy is
$\delta$-discriminative, the belief will converge to either $0$ or
$1$ exponentially fast. Therefore by restricting the MBP policy to
be $\delta$-discriminative (referred to CMBP in \cite{BDPP:10})
finite regret can be achieved.

%Another variation of MBP called adaptive MBP (AMBP) is also
%considered. In AMBP whenever the belief update is too close to its
%previous value $|q(t) - q(t-1)| < \epsilon$, the policy picks a
%predetermined experiment price that will assure learning. This
%policy achieves finite expected regret as well.

\section{The Non-Bayesian Approach}\label{non-Bayesian}

In this section we present our main result developed within the
non-Bayesian framework. We first consider $N=2$ and leave the
general case to Sec.~\ref{General}. We assume no prior probabilistic
knowledge on which demand model may pertain. We formulate this
problem as a two-armed bandit problem within the non-Bayesian
framework as follows. Let
\begin{eqnarray}\nn
p_k^* = \arg\max_{p \in [l, u]}p \rho_k(p).~~~~~~~~~~~~~k= 0, 1.
\end{eqnarray}
Arm $k$, $k=0,1$, is defined as the price $p_k^*$. If the underlying
demand model is $\rho_k(.)$, arm $k$ is the better arm. Activating
arm $k$ is defined as offering price $p_k^*$ to the costumers. The
reward random variable $X_k$ for arm $k$ is defined as the revenue
by proposing price $p_k^*$ at each time. The reward mean of arm $k$
is $\mbbE[X_k] = p_k^* \rho_i(p_k^*)$ when $\rho_i(.)$ is the
unknown underlying demand model. Throughout this paper we assume
that no $p_k^*$ is an uninformative price under any underlying
demand model, meaning that for both $k = 0, 1$,
\begin{eqnarray}\nonumber
\rho_1(p_k^*) \neq \rho_0(p_k^*).
\end{eqnarray}
This assumption is needed in order to achieve finite regret.

 Activating each arm (offering price $p_k^*$) gives i.i.d.
realizations of random reward $X_k$. Since both arms share the same
underlying demand model $\rho_i(.)$, arms are correlated.In other
words, observations from one arm also provide information on the
quality of the other arm.

 We define regret or revenue loss as the following:
\begin{eqnarray*}
\Delta_i= T p_i^*\rho_i(p_i^*)-\sum_{k=1}^2
\{p_{k}^*\rho_i(p_k^*)\mathbb{E}[T_k]\}
\end{eqnarray*}
where $T_k$ is the number of times that arm $k$ is selected, and
$\rho_i(.)$ is the true underlying demand model.

\subsection{The LRT Policy}\label{LRT}

In this section, we propose the LRT policy and establish its finite
regret. For each arm $k$, let $Y_k$ denote the seller's observation
when it activates arm $k$. It is a binary random variable with
mean$\rho_i(p_k^*)$.
 The LRT policy at each time step $t$ is a function $a_t$ mapping from
the observation space $\{y_1,y_2,\dots,y_i,\dots,y_{t-1}\}$ to the
action space $\{0,1\}$ (arms of the bandit). Specifically, in the
first step $t=1$, the LRT policy chooses an arm $k \in \{0,1\}$ by
flipping a fair coin. For each $t
> 1$, let
\begin{eqnarray}\label{lrtratio}
L(t) = \frac{1}{t}\sum_{j=1}^{t}\log{\frac{f_1(y_j)}{f_0(y_j)}},
\end{eqnarray}
where $f_i(y_j) = \Pr\{Y_{a_j} = y_j| \rho(.) = \rho_i(.)\}$ is the
probability of observing $y_j$ when action $a_j$ is chosen if the
underlying demand model were $\rho_i(.)$. Then the LRT policy at
each time step $t$ decides which arm to activate based on the
following
\begin{eqnarray}\label{LRTpolicy}
L(t-1)\left.
\begin{array}{l}
{a_t=1}\\
\quad\gtreqless\\
{a_t=0}
\end{array}\right.
0,
\end{eqnarray}
where $a_t$ denotes the action at time $t$. It is easy to see that
$L(t)$ can be updated recursively as
\begin{eqnarray}\nn
L(t)=\frac{1}{t}[(t-1)L(t-1)+ \log{\frac{f_1(y_t)}{f_0(y_t)}}].
\end{eqnarray}
The LRT policy is based on the maximum likelihood detector. In the
following theorem we show that the LRT policy has finite regret.

\begin{theorem}\label{Thm LRT}
The LRT policy achieves a bounded regret.
\end{theorem}
\begin{proof}
The proof is a special case of the proof of Theorem~\ref{Thm XLRT}
by setting $\eta_0 = \eta_1 = 0$
\end{proof}

\subsection{The XLRT Policy}\label{XLRT} In this section we propose a
generalization of the LRT policy to improve its regret performance.
Based on the underlying detection nature of the problem, we
generalize the LRT policy by introducing an exploration price. We
aim to choose a price as our exploration price in order to
accelerate the learning of the underlying demand model.

In particular, this exploration price should be chosen such that
$\rho_1(p)$ and $\rho_0(p)$
 are most easily distinguished from random observations. Recognizing the detection nature of the problem,
  we adopt the Chernoff distance~\cite{Chernoff:52} which measures the distance between two distributions by the
   asymptotic exponential decay rate of the probability of detection errors. Specifically, for two probability
   density functions $f_0$ and $f_1$, the Chernoff distance is given by
\begin{eqnarray}\label{Chernoff-distance}
\mathcal{C}(f_0,f_1) &=& \underset{0 \leq t \leq 1}{\max} - \log
\mu(t),\\\nn \mu(t) &=& \int [f_0(x)]^{1-t} [f_1(x)]^t dx.
\end{eqnarray}
Since calculating the Chernoff distance involves an optimization
step, obtaining an analytical solution can be tedious. Johnson and
Sinanovic in \cite{Johnson&Sinanovic} stated that the harmonic
average of the Kullback-Leibler
divergences~\cite{Kullback&Leibler:51} of $f_0$ and $f_1$ can be a
good approximation of the Chernoff distance which is easy to
calculate. Namely,
\begin{eqnarray}
\hat{\mathcal{C}}(f_0,f_1) = \frac{1}{\frac{1}{I(f_0 ||
f_1)}+\frac{1}{I(f_1 || f_0)}},
\end{eqnarray}
 where the Kullback-Leibler divergence of $f_0$ with respect to $f_1$ is given by
\begin{eqnarray}
I(f_0 || f_1) = \int f_0(x) \log \frac{f_0(x)}{f_1(x)} dx.
\end{eqnarray}
Note that Kullback-Leibler divergence is not symmetric and the
Chernoff distance can be thought of as the symmetrized
Kullback-Leibler divergence.

 The exploration price is thus chosen as the price that maximizes the Chernoff distance $\mathcal{C}(\rho_1(p),\rho_0(p))$.
 Observations obtained by this price are the most informative in
distinguishing the two possible candidates of the demand model. The
exploration price is offered when the log-likelihood ratio $L(t)$ is
close to $0$, \ie when it is most uncertain which demand model
pertains. This is done by introducing two thresholds $\eta_1$, and
$-\eta_0$ instead of the single threshold $0$ in the LRT policy (see
\eqref{LRTpolicy}). The resulting policy is referred to XLRT as
detailed below.

Set the exploration price as
\begin{eqnarray}
p_x = \arg\max_{p \in [l,u]} \{\mathcal{C}(\rho_1(p),\rho_0(p))\}.
\end{eqnarray}

%
%where $I(\rho_1(p), \rho_0(p))$ is the Kullback-Leibler divergence
%between two Bernoulli random variables with means $\rho_1(p)$, and
%$\rho_0(p)$ respectively.
%\begin{eqnarray}\nn
%I(\rho_1(p), \rho_0(p)) = \rho_1(p) \log \frac{\rho_1(p)}{\rho_0(p)}
%+ (1-\rho_1(p)) \log \frac{1 -\rho_1(p)}{1 - \rho_0(p)}.
%\end{eqnarray}

%Since the Kullback-Leibler divergence is not symmetric, the
%arithmetic mean of $I(\rho_1(p), \rho_0(p))$ and $I(\rho_0(p),
%\rho_1(p))$ is chosen as the metric to determine the exploration
%price.

 At each step $t$:
\begin{eqnarray}\label{XLRTpolicy} p_t = \left\{\begin{array}{ccc}
p^*_0 & L(t-1) < -\eta_0\\
p_x & -\eta_0 \leq L(t-1) \leq \eta_1\\
p^*_1 & \eta_1 \leq L(t-1)\\
\end{array}
\right.,
\end{eqnarray}
where $L(t-1)$ is the log likelihood ratio given in~\eqref{lrtratio}
and the two thresholds should satisfy the following conditions:
\begin{eqnarray}\nn \eta_1 <
\min\{I(\rho_1(p_1^*), \rho_0(p_1^*)), I(\rho_1(p_x), \rho_0(p_x)),
I(\rho_1(p_0^*),\rho_0(p_0^*))\},\\\nn \eta_0 <
\min\{I(\rho_0(p_1^*), \rho_1(p_1^*)), I(\rho_0(p_x), \rho_1(p_x)),
I(\rho_0(p_0^*), \rho_1(p_0^*))\}.
\end{eqnarray}

This policy differs from the LRT policy only when the likelihood
ratio $L(t)$ is close to zero (indicating a greater degree of
uncertainty on the underlying demand model). At such time instants,
XLRT chooses the price $p_x$ that is most informative in learning
the underlying demand model. Simulation examples demonstrate that
XLRT policy improves the performance of the LRT policy (see
Sec.~\ref{simulation}).

\begin{theorem}\label{Thm XLRT}
The XLRT policy achieves a bounded regret.
\end{theorem}

\begin{proof}
Without loss of generality, we assume that the underlying demand
model is $\rho_1(.)$. Let $M_e$ denote the expected number of times
that the XLRT policy chooses the non-optimal price.
\begin{eqnarray}\nonumber
M_e =
\mathbb{E}[\sum_{t=1}^{T}\mathbf{1}\{\frac{1}{t}\sum_{j=1}^t\log{\frac{f^{a_j}_1(y_j)}{f^{a_j}_0(y_j)}}<
\eta_1\}] &\leq&
\sum_{t=1}^{\infty}\mathbb{E}[\mathbf{1}\{\frac{1}{t}\sum_{j=1}^t\log{\frac{f^{a_j}_1(y_j)}{f^{a_j}_0(y_j)}}<
\eta_1\}]\\\nonumber
&=&\sum_{t=1}^{\infty}\Pr\{\frac{1}{t}\sum_{j=1}^t\log{\frac{f^{a_j}_1(y_j)}{f^{a_j}_0(y_j)}}<
\eta_1\},
\end{eqnarray}
where $\mathbf{1}\{.\}$ is the indicator function. Note that both
the action $a_j$ and the observation $y_j$ at time $j$ are random
variables. To simplify the notation, let $Z_j^{a_j} =
\log{\frac{f^{a_j}_1(y_j)}{f^{a_j}_0(y_j)}}$. We then have
\begin{eqnarray}\nonumber
\Pr\{\frac{1}{t}\sum_{j=1}^t\log{\frac{f^{a_j}_1(y_j)}{f^{a_j}_0(y_j)}}<\eta_1\}
&=& \Pr\{\frac{1}{t}\sum_{j=1}^{t} Z_j^{a_j} < \eta_1\}\\\nonumber
&=& \mathbb{E}[\Pr\{\frac{1}{t}\sum_{j=1}^{t} Z_j^{a_j} < \eta_1 |
\{a_j\}_{j=1}^t \}].
\end{eqnarray}
where the expectation is over the action sequence $\{a_j\}_{j=1}^t$.

Notice that the action $a_j$ at each time $j$ takes three possible
values: $\{p_0^*, p_x, p_1^*\}$. We label these three prices as
$p(1) = p_0^*$, $p(2) = p_x$, and $p(3)= p_1^*$, and let $Z_j^k =
\log
\frac{f_1^k(y_j)}{f_0^k(y_j)}=\log\frac{f_1^{a_j=p(k)}(y_j)}{f_0^{a_j=p(k)}(y_j)}$
denote the log-likelihood ratio conditioned on that price $p(k)$
($k=1,2,3$) is chosen. It is easy to see that for each $k = 1, 2,
3$, $Z_j^k$'s are i.i.d. binary random variables taking the values
$\{\log{\frac{f^k_1(0)}{f^k_0(0)}},
\log{\frac{f^k_1(1)}{f^k_0(1)}}\}$ with probabilities
$1-\rho_1(p(k))$ and $\rho_1(p(k))$, respectively. Since the
underlying demand model is $\rho_1(.)$, we have
\begin{eqnarray}\nonumber
\mathbb{E}[Z_j^k] =
\mathbb{E}_1[\log{\frac{f^k_1(y_j)}{f^k_0(y_j)}}] =
I(\rho_1(p(k))||\rho_0(p(k))) > 0,
\end{eqnarray}
where $I(\alpha || \beta)$ is the Kullback-Leibler divergence of two
Bernoulli random variables with means $\alpha$, and $\beta$.

Let $m_k = \min \{\log{\frac{f^k_1(0)}{f^k_0(0)}},
\log{\frac{f^k_1(1)}{f^k_0(1)}}\}$, $M_k = \max
 \{\log{\frac{f^k_1(0)}{f^k_0(0)}},
\log{\frac{f^k_1(1)}{f^k_0(1)}}\}$, $m = \min \{m_1, m_2, m_3\}$,
and $M = \max \{M_1, M_2, M_3\}$. $Z_j^k$'s  are independent and
bounded in the interval $[m, M]$ for all $k= 1,2,3$.

Let $\Theta_k^t = \{j: a_j = k, j \leq t\}$ for $k \in \{1, 2, 3\}$,
then
\begin{eqnarray}\nonumber
\Pr\{\frac{1}{t}\sum_{j=1}^{t} Z_j^{a_j} < \eta_1\}  &=&
\mathbb{E}[\Pr\{\frac{1}{t}\sum_{j=1}^{t} Z_j^{a_j} < \eta_1 |
\{a_j\}_{j=1}^t \}]\\\nonumber &=&
\mathbb{E}[\Pr\{\frac{1}{t}\sum_{j=1}^{t} Z_j^{a_j} < \eta_1 |
\{a_j\}_{j=1}^t, \cup_{k=1}^3 \Theta_k^t \}]\\\nonumber &=&
\mathbb{E}[\Pr\{\frac{1}{t}\sum_{k=1}^3\sum_{h \in \Theta_k^t}
Z_h^{k} < \eta_1 | \{a_j\}_{j=1}^t, \cup_{k=1}^3 \Theta_k^t \}]
\\\nonumber &=& \mathbb{E}[\Pr\{\frac{1}{t}\sum_{k=1}^3\sum_{h \in \Theta_k^t}
Z_h^{k} < \eta_1 | \cup_{k=1}^3 \Theta_k^t \}] \\\nonumber &=&
\mathbb{E}[\Pr\{\frac{1}{t}\sum_{k=1}^3 \sum_{h \in \Theta_k^t}
(Z_h^{k} - \mathbb{E}[Z_h^{k}]) < - (\frac{1}{t}\sum_{k=1}^3\sum_{h
\in \Theta_k^t}\mathbb{E}[Z_h^{k}] - \eta_1) | \cup_{k=1}^3
\Theta_k^t \}].
\end{eqnarray}
The second equality is because the sigma-algebra generated by
$\cup_{k=1}^3 \Theta_k^t$ is subset of the  sigma-algebra generated
by $\{a_j\}_{j=1}^t$. The fourth equality is because the conditional
event $\{\frac{1}{t}\sum_{k=1}^3\sum_{h \in \Theta_k^t} Z_h^{k} <
\eta_1 | \cup_{k=1}^3 \Theta_k^t\}$ is independent of
$\{a_j\}_{j=1}^t$.

 Note that conditioned on
$\cup_{k=1}^3 \Theta_k^t$, the sum $\sum_{k=1}^3\sum_{h \in
\Theta_k^t} (Z_h^{k} - \mathbb{E}[Z_h^{k}])$ is the sum of $t$
independent zero mean random variables. The rest of the proof is
based on the following Hoeffding's inequality.

\emph{Hoeffding's inequality}: Let $X_i$'s be independent random
variables that are almost surely bounded in $[m_i, M_i]$, i.e., $\Pr
\{X_i \in [m_i, M_i]\} = 1$, then for some $\alpha > 0$,
\begin{eqnarray}
\Pr \{ \frac{1}{t} \sum_{i=1}^t (X_i - \mathbb{E}[X_i]) < -\alpha\}
\leq \exp\{-\frac{2\alpha^2 t^2}{\sum_{i=1}^t (M_i-m_i)^2}\}
\end{eqnarray}

Therefore Hoeffding's inequality can be used for the conditional
probability inside the expectation for independent random variables
$Z_h^k$'s that are bounded in $[m, M]$ to obtain
\begin{eqnarray}\nonumber
 & &\Pr\{\frac{1}{t}\sum_{k=1}^3\sum_{h\n \Theta_k^t} (Z_h^k - \mathbb{E}[Z_h^k])
< - (\frac{1}{t}\sum_{k=1}^3\sum_{h \in
\Theta_k^t}\mathbb{E}[Z_h^{k}] - \eta_1) | \cup_{k=1}^3 \Theta_k^t
\}\\\nonumber &\leq& \Pr\{\frac{1}{t}\sum_{k=1}^3\sum_{h \in
\Theta_k^t} (Z_h^k - \mathbb{E}[Z_h^k]) <- (\min_{k
=1,2,3}\{I(\rho_1(p(k))||\rho_0(p(k)))\} - \eta_1) | \cup_{k=1}^3
\Theta_k^t \}\\\nonumber &=& \Pr\{\frac{1}{t}\sum_{k=1}^3\sum_{h \in
\Theta_k^t} (Z_h^k - \mathbb{E}[Z_h^k]) <- a | \cup_{k=1}^3
\Theta_k^t\} \leq \exp\{-\frac{2a^2t}{(M - m)^2}\} =
\exp\{-\frac{t}{C}\},
\end{eqnarray}
where $a = \underset{{k
=1,2,3}}{\min}\{I(\rho_1(p(k))||\rho_0(p(k)))\} - \eta_1$ and $C =
\frac{(M - m)^2}{2a^2}$.
 Recall that the thresholds are chosen such that $\eta_1 < \underset{{k =1,2,3}}{\min}\{I(\rho_1(p(k))||\rho_0(p(k)))\}$
based on definition. Therefore $a > 0$. Hence
\begin{eqnarray}\nonumber
\Pr\{\frac{1}{t}\sum_{j=1}^t Z_j^{a_j}< \eta_1\} =
\mathbb{E}[\Pr\{\frac{1}{t}\sum_{j=1}^{t} Z_j^{a_j} < \eta_1 |
\{a_j\}_{j=1}^t \}] &\leq&
\mathbb{E}[\Pr\{\frac{1}{t}\sum_{k=1}^3\sum_{h \in \Theta_k^t}
(Z_h^k - \mathbb{E}[Z_h^k]) <- a  | \cup_{k=1}^3
\Theta_k^t\}]\\\nonumber &\leq& \mathbb{E}[\exp\{-\frac{t}{C}\}] =
\exp\{-\frac{t}{C}\}.
\end{eqnarray}
 Hence
\begin{eqnarray}\nonumber
M_e \leq
\sum_{t=1}^{\infty}\Pr\{\frac{1}{t}\sum_{j=1}^t\log{\frac{f^{a_j}_1(y_j)}{f^{a_j}_0(y_j)}}<\eta_1\}
\leq  \sum_{t=1}^{\infty} \exp\{-\frac{t}{C}\} \leq
\int_{t=0}^{\infty} \exp\{-\frac{t}{C}\} = C < \infty.
\end{eqnarray}
It shows that the expected number of times the non-optimal price is
chosen is bounded by a finite number ${C}$. Therefore regret for the
LRT policy is bounded above by ${C}$ multiplied by a constant.
\end{proof}

\section{Extension to Finite Space Demand Uncertainty}\label{General}
In this section we extend the LRT and XLRT policies to handle
finite-space demand uncertainty where the underlying
 demand model is taken from a set $\{\rho_0(.), \ldots,
\rho_{N-1}(.)\}$ with an arbitrary finite cardinality $N$.

This problem can be formulated as a multi-armed bandit problem in
the same way as in section \ref{non-Bayesian}. Arm $k$, $k=0,
\ldots, N-1$, is defined as the price $p_k^*$ where
\begin{eqnarray}\nn
p_k^* = \arg\max_{p \in [l, u]}p \rho_k(p).~~~~~~~~~~~~~k= 0,
\ldots, N-1.
\end{eqnarray}
 If the
underlying demand model is $\rho_k(.)$, arm $k$ is the best arm. As
mentioned earlier in order to achieve finite regret we assume that
no optimal price is uninformative under any demand model. In other
words for all $j, h , k \in \{0, \ldots, N-1\}$, $j \neq h$,
\begin{eqnarray}\nonumber
\rho_j(p_k^*) \neq \rho_h(p_k^*).
\end{eqnarray}
%Activating arm $k$ is defined as offering price $p_k^*$ to the
%costumers. The reward random variable $X_k$ for arm $k$ is defined
%as the revenue by proposing price $p_k^*$ at each time. The reward
%mean of arm $k$ is $\mbbE[X_k] = p_k^* \rho_i(p_k^*)$ when
%$\rho_i(.)$ is the underlying demand model.
%Activating each arm (offering prices $p_k^*$) gives i.i.d.
%realizations of reward random variables $X_k$. Since both arms share
%the same underlying demand model $\rho_i(.)$, arms are correlated.
 The regret is also defined in a similar way as
\begin{eqnarray*}
\Delta_i= T p_i^*\rho_i(p_i^*)-\sum_{k=0}^{N-1}
\{p_{k}^*\rho_i(p_k^*)\mathbb{E}[T_k]\}
\end{eqnarray*}
where $T_k$ is the number of times that arm $k$ is selected.

We present the extended versions of LRT and XLRT (referred to as
ELRT and EXLRT) policies for an arbitrary $N$ number of potential
demand models. We also show that the proposed ELRT policy achieves
finite regret. Similarly, for each arm $k$, define the binary random
variable $Y_k \in \{0, 1\}$ with mean $\rho_i(p_k^*)$. $Y_k$ is the
random variable indicating the seller's observation when it
activates arm $k$ (\ie offers price $p_k^*$). The ELRT policy at
each time step $t$ is a function $a_t$ mapping from the observation
space $\{y_1,y_2,\dots,y_i,\dots,y_{t-1}\}$ to the action space
$\{0, \ldots, N-1\}$ (arms of the bandit). For the first step $t=1$,
choose an arm uniformly from $k \in \{0, \ldots, N-1\}$. For the
time step $t
> 1$, let
\begin{eqnarray}
L_{i,h}(t) =
\frac{1}{t}\sum_{j=1}^{t}\log{\frac{f_i(y_j)}{f_h(y_j)}},
\end{eqnarray}
where $f_i(y_j) = \Pr\{Y_{a_j} = y_j| \rho(.) = \rho_i(.)\}$ is the
probability of observing $y_j$ when action $a_j$ is chosen if the
underlying demand model is $\rho_i(.)$.

The ELRT policy at each time step $t$ selects arm $a_t = k$ for
which
%Let $\hat{i}(1)=0$ and $h(n)=n+1$, then do the following for
%$n=0,\ldots,
% N-2$,
\begin{eqnarray}
L_{k,j}(t-1) > 0,~~~ \forall j \in \{0, \ldots, N-1\}, j \neq k.
\end{eqnarray}
%\begin{displaymath}
%L_{\hat{i}(n),h(n)}(t-1)\left.
%\begin{array}{l}
%{\hat{i}(n+1)=\hat{i}(n)}\\
%\quad\gtreqless\\
%{\hat{i}(n+1)=h(n)}
%\end{array}\right.
%0.
%\end{displaymath}
%Then $a_t = \hat{i}(N)$ where $a_t$ denotes the action at time $t$.

%The LRT policy chooses arm $a_t$ that the likelihood function is
%maximized if $\rho_{a_t}$ is the true demand model. In other words,
%the LRT policy at time $t$ chooses arm $a_t$ that its likelihood
%function $L_{a_t,h} > 0$, for $0\leq h \leq N - 1 $, $h \neq a_t$.

%The GLRT policy is based on the maximum likelihood detector. It aims
%to detect the underlying demand model and maximize the revenue at
%the same time. In the following theorem we show that the GLRT policy
%has finite regret.

\begin{theorem}\label{Thm GLRT}
The ELRT policy achieves a bounded regret.
\end{theorem}
\begin{proof}
The proof is a direct generalization of the Theorem~\ref{Thm LRT}
and is given in Appendix for completeness.
\end{proof}
%\subsection{GXLRT Policy} In this section we generalize the XLRT policy proposed in section
%\ref{XLRT}.
Similarly, we can introduce exploration prices to improve the rate
of learning. Since there are now $N$ possible demand models, there
are $N-1$ possible exploration prices. At each time, when the
likelihood ration between the two models that are detected as the
most likely is close to $0$, the exploration price defined by the
Chernoff distance between these two demand models is offered.
Specifically, let
\begin{eqnarray}
p^{i,h}_x = \arg\max_{p \in [l,u]}
\{\mathcal{C}(\rho_i(p),\rho_h(p))\}.
\end{eqnarray}
At each time step $t$ the policy first finds the two most probable
demand models $d_1$ and $d_2$, where model $d_1$ satisfies
\begin{eqnarray}
L_{d_1,j}(t-1) > 0~~~~\forall j \in \{0, \ldots, N-1\}, j \neq d_1,
\end{eqnarray}
and model $d_2$ satisfies
\begin{eqnarray}
L_{d_2,j}(t-1) > 0~~~~\forall j \in \{0, \ldots, N-1\}, j \neq d_1,
j \neq d_2,
\end{eqnarray}
Then
\begin{eqnarray}\nn p_t = \left\{\begin{array}{ccc}
p^*_{d_1} & L_{d_1,d_2}(t-1) > \eta_{d_1,d_2}\\
p^{d_1,d_2}_x & 0 \leq L_{d_1,d_2}(t-1) \leq \eta_{d_1,d_2}\\
\end{array}
\right.,
\end{eqnarray}
where the threshold $\eta_{d_1,d_2}$ is chosen to satisfy the
following conditions:
\begin{eqnarray}\nn \eta_{d_1,d_2} <
\min \{I(\rho_{d_1}(p_{d_1}^*), \rho_{d_2}(p_{d_1}^*)),
I(\rho_{d_1}(p^{d_1,d_2}_x), \rho_{d_2}(p^{d_1,d_2}_x)),
I(\rho_{d_1}(p_{d_2}^*),\rho_{d_2}(p_{d_2}^*))\}.
\end{eqnarray}

%This policy differs from the LRT policy only when the metric
%$L_{d_1,d_2}(t)$ is fluctuating around zero indicating that the
%policy is not certain on which optimal price to choose. When such
%situation happens, XLRT offers the price that can help the seller
%learn the underlying demand model faster. Simulation examples
%demonstrate that XLRT policy improves the performance of the LRT
%policy.

%\section{Simulation Results}
%In this section we study the performance of the proposed policies
%LRT and XLRT.
%
%\begin{figure}[htp]
%\begin{center}
%\scalefig{0.5}\epsfbox{figures/mbpVSlrt.eps}\caption{LRT vs.
%CMBP}\label{fig:lrt}
%\end{center}
%\end{figure}
%
%Fig.~\ref{fig:lrt} shows the comparison of the LRT policy with the
%constrained MBP (CMBP) policy proposed in \cite{BDPP:10}. We used
%the demand models adopted in \cite{BDPP:10} with $\rho_0(p) = 1.4 -
%0.9 p$ and $\rho_1(p) = 0.8 - 0.3 p$ for the price range of $[0.5,
%1.5]$. The initial belief for the CMBP policy is set to be $0.5$ in
%order to have a fair comparison. As shown in Fig.~\ref{fig:lrt}, the
%LRT policy which uses neither any prior knowledge on the underlying
%demand model nor complete knowledge of the demand curves outperforms
%the CMBP policy.
%\begin{figure}[htp]
%\begin{center}
%%\psfrag{T0}[c]{\Large  $T_0$} \psfrag{T1}[c]{\Large $T_1$}
%\scalefig{0.5}\epsfbox{figures/xlrt.eps}\caption{XLRT vs.
%LRT}\label{fig:xlrt}
%\end{center}
%\end{figure}
% Fig.~\ref{fig:xlrt} shows the comparison of XLRT policy with LRT.
%We observe that introducing the exploration price in XLRT has
%considerably improved the regret.
\section{Simulations}\label{simulation}
In this section we study the performance of the proposed policies.
In Fig.~\ref{fig:cmbp1} and Fig.~\ref{fig:cmbp2} shows the
comparison of the LRT policy with the constrained MBP (CMBP)proposed
in \cite{BDPP:10} under a binary demand model sets chosen in
\cite{BDPP:10}. In these two examples, we consider the same demand
models used in the simulation studies in \cite{BDPP:10}.
 In particular, in Fig.~\ref{fig:cmbp1},
$\rho_0(p) = 1.4 - 0.9 p$ and $\rho_1(p) = 0.8 - 0.3 p$ for the
price range of $[0.5, 1.5]$, and in Fig.~\ref{fig:cmbp2}, $\rho_0(p)
=\frac{1}{1 + \exp(-10 + 10 p)}$ and $\rho_1(p) = \frac{1}{1 +
\exp(-1 + 0.5 p)}$ for the price range of $[0, 4]$. The initial
belief is chosen to be $q_0 = 0.5$. We observe that the proposed LRT
policy outperforms CMBP in both cases.

\begin{figure}[htp]
\begin{center}
%\psfrag{T0}[c]{\Large  $T_0$} \psfrag{T1}[c]{\Large $T_1$}
\scalefig{.7}\epsfbox{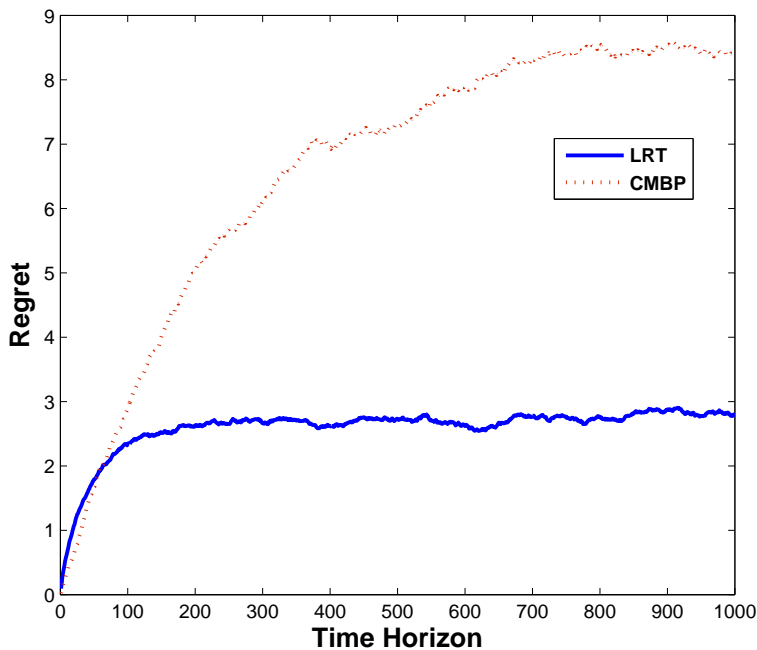}\caption{CMBP vs. LRT: Case
$1$}\label{fig:cmbp1}
\end{center}
\end{figure}

\begin{figure}[htp]
\begin{center}
%\psfrag{T0}[c]{\Large  $T_0$} \psfrag{T1}[c]{\Large $T_1$}
\scalefig{.7}\epsfbox{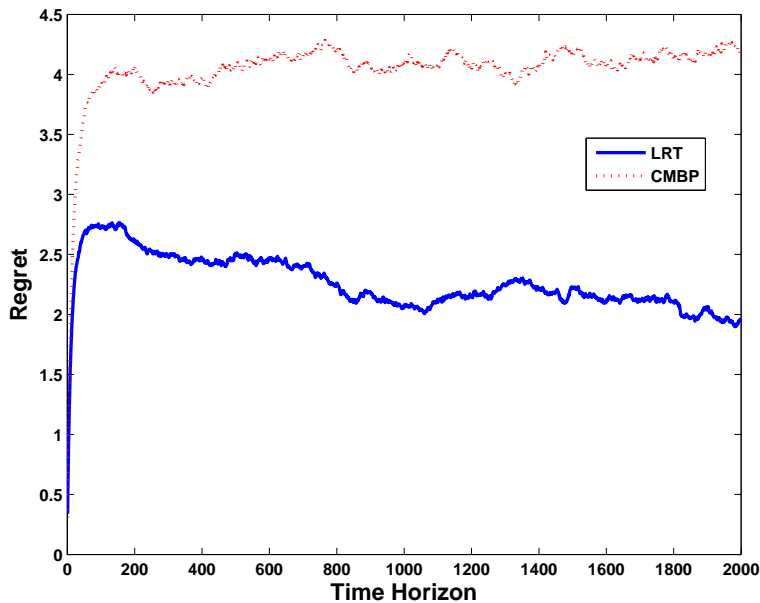}\caption{CMBP vs. LRT: Case
$2$}\label{fig:cmbp2}
\end{center}
\end{figure}
Fig.~\ref{fig:xlrt0} and Fig.~\ref{fig:xlrt1} show the comparison of
XLRT policy with LRT for the demand models $\rho_0(p) = 1.4 - 0.9 p$
and $\rho_1(p) = 0.8 - 0.3 p$ when the underlying demand model is
$\rho_0(.)$ and $\rho_1(.)$ respectively. We observe that
introducing the exploration price in XLRT considerably improves the
regret.
\begin{figure}[htp]
\begin{center}
%\psfrag{T0}[c]{\Large  $T_0$} \psfrag{T1}[c]{\Large $T_1$}
\scalefig{.7}\epsfbox{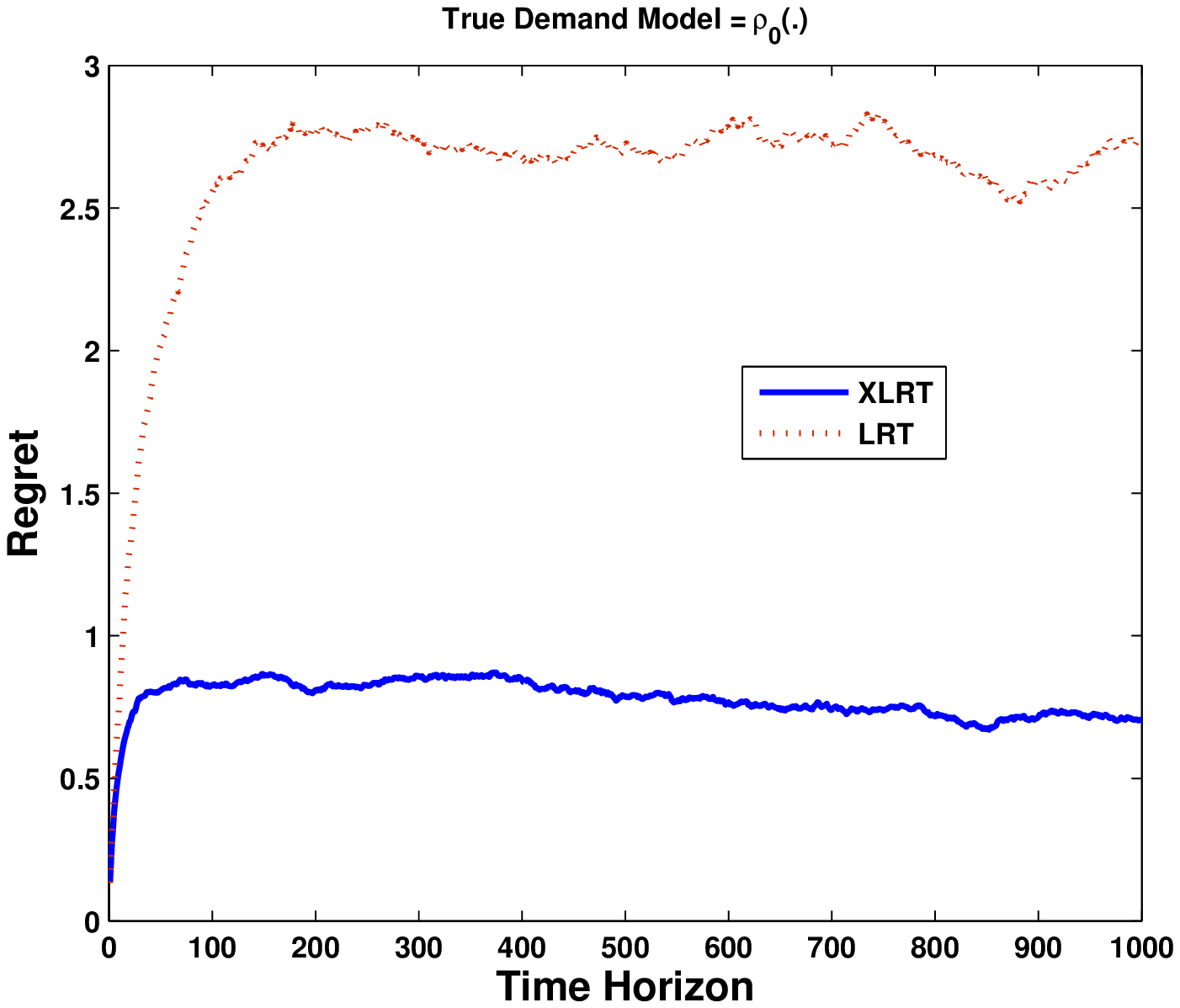}\caption{XLRT vs. LRT when
$\rho = \rho_0$}\label{fig:xlrt0}
\end{center}
\end{figure}
\begin{figure}[htp]
\begin{center}
%\psfrag{T0}[c]{\Large  $T_0$} \psfrag{T1}[c]{\Large $T_1$}
\scalefig{.7}\epsfbox{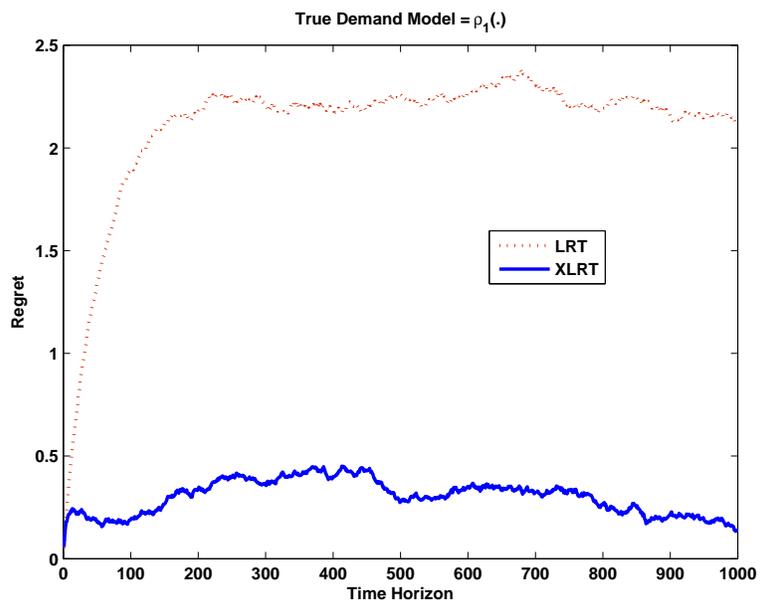}\caption{XLRT vs. LRT when
$\rho = \rho_1$}\label{fig:xlrt1}
\end{center}
\end{figure}

\section{conclusion}
The dynamic pricing problem when the underlying demand model is
unknown to the seller is considered where the demand model takes one
of $N$ possible forms where $N$ is any arbitrary number. A
non-Bayesian approach in which no prior knowledge on which demand
model is governing the market is studied. The problem is formulated
as a multi-armed bandit with correlated arms. A policy based on the
likelihood ratio test is developed that achieves finite regret. An
generalization of this policy is proposed by introducing an
exploration price that helps the seller to learn the underlying
demand model faster and improve the regret performance.

\section*{Appendix: Proof of Theorem \ref{Thm GLRT}}
Without loss of generality, we assume that the underlying demand
model is $\rho_i(.)$.

 Let $M_e$ denote the expected number of times that the ELRT policy chooses the non-optimal price.
\begin{eqnarray}\nonumber
M_e &\leq& \sum_{h=0,h\neq i}^{N-1}
\mathbb{E}[\sum_{t=1}^{T}\mathbf{1}\{\frac{1}{t}\sum_{j=1}^t\log{\frac{f^{a_j}_i(y_j)}{f^{a_j}_h(y_j)}}<0\}]\\\nonumber
&\leq& \sum_{h=0,h\neq i}^{N-1}
\sum_{t=1}^{\infty}\mathbb{E}[\mathbf{1}\{\frac{1}{t}\sum_{j=1}^t\log{\frac{f^{a_j}_i(y_j)}{f^{a_j}_h(y_j)}}<0\}]\\\nonumber
&=& \sum_{h=0,h\neq i}^{N-1}
\sum_{t=1}^{\infty}\Pr\{\frac{1}{t}\sum_{j=1}^t\log{\frac{f^{a_j}_i(y_j)}{f^{a_j}_h(y_j)}}<0\}.
\end{eqnarray}
 Note that both
the action $a_j$ and the observation $y_j$ at time $j$ are random
variables. To simplify the notation consider a fixed $h$ and let
$Z_j^{a_j} = \log{\frac{f^{a_j}_i(y_j)}{f^{a_j}_h(y_j)}}$. At each
time $j$, the policy $a_j$ chooses one of the $N$ possible prices
$p_k^*$'s, $k = 0, \ldots, N-1$. For each chosen price $p_k^*$,
$y_j$ is a Bernoulli random variable. Let $Z_j^k = \log
\frac{f_i^k(y_j)}{f_h^k(y_j)}=\log\frac{f_i^{a_j= k}(y_j)}{f_h^{a_j=
k}(y_j)}$ denote the log-likelihood ratio conditioned on that price
$p_k^*$ is chosen. It is easy to see that for each $k = 0, \ldots,
N-1$, $Z_j^k$'s are i.i.d. binary random variables taking the values
$\{\log{\frac{f^k_i(0)}{f^k_h(0)}},
\log{\frac{f^k_i(1)}{f^k_h(1)}}\}$ with probabilities $1 -
\rho_i(p_k^*)$ and $\rho_i(p_k^*)$ respectively. Since the
underlying demand model is $\rho_i(.)$, we have
\begin{eqnarray}\nonumber
\mathbb{E}[Z_j^k] =
\mathbb{E}_i[\log{\frac{f^k_i(y_j)}{f^k_h(y_j)}}] =
I(\rho_i(p_k^*)||\rho_h(p_k^*)) > 0.
\end{eqnarray}
%where $I(\alpha || \beta)$ is the Kullback Leibler divergence of two
%Bernoulli random variables with means $\alpha$, and $\beta$.
Let $m_k^h = \min \{\log{\frac{f^k_i(0)}{f^k_h(0)}},
\log{\frac{f^k_i(1)}{f^k_h(1)}}\}$, $M_k^h = \max
 \{\log{\frac{f^k_i(0)}{f^k_h(0)}},
\log{\frac{f^k_i(1)}{f^k_h(1)}}\}$, $m_h = \min \{m_1^h, \ldots,
m_N^h\}$, and $M_h = \max \{M_1^h, \ldots, M_N^h\}$. $Z_j^k$'s are
independent and bounded in the interval $[m_h, M_h]$. Following the
steps in the proof of Theorem \ref{Thm XLRT} for $\eta_0 = \eta_1 =
0$ and conditioning on the event $\cup_{k=0}^{N-1} \Theta_k^t$ as
the sufficient statistic for the action history $\{a_j\}_{j=1}^t$
one can get
\begin{eqnarray}\nonumber
\Pr\{\frac{1}{t}\sum_{j=1}^t\log{\frac{f^{a_j}_i(y_j)}{f^{a_j}_h(y_j)}}<0\}
\leq \exp\{-\frac{t}{C_h}\}
\end{eqnarray}
where $C_h = \frac{(M_h - m_h)^2}{2a_h^2}$ and $a_h = \min_{k =0,
\ldots, N-1}\{I(\rho_i(p_k^*)||\rho_h(p_k^*))$. Hence
\begin{eqnarray}\nonumber
M_e  \leq  \sum_{h=0,h\neq i}^{N-1}
\sum_{t=1}^{\infty}\Pr\{\frac{1}{t}\sum_{j=1}^t\log{\frac{f^{a_j}_i(y_j)}{f^{a_j}_h(y_j)}}<0\}
&\leq&  \sum_{h=0,h\neq i}^{N-1} \sum_{t=1}^{\infty}
\exp\{-\frac{t}{C_h}\}\\\nonumber &\leq&  \sum_{h=0,h\neq i}^{N-1}
\int_{t=0}^{\infty} \exp\{-\frac{t}{C_h}\} \leq (N-1) C < \infty,
\end{eqnarray}
where $C= \max_{h=0, h\neq i}^{N-1} \{C_h\}$.

 It shows that the expected
number of times the non-optimal price is chosen is bounded by a
finite number ${(N-1)C}$. Therefore regret for the LRT policy is
bounded above by ${(N-1)C}$ multiplied by a constant.

%%%%%%%%%% References %%%%%%%%%%%%%%%%%%%%%%%%%%%%%%%%%%%%%%%%%%%%%%%%%%
\bibliographystyle{ieeetr}

\end{document}

%% file: macros.tex
\def\boxit#1{\vbox{\hrule\hbox{\vrule\kern3pt
        \vbox{\kern3pt#1\kern3pt}\kern3pt\vrule}\hrule}}

\def\reals{ { {\rm  I \kern-0.15em R }  } }
\def\complex{ {\,{{\rm C} \kern-0.50em \raise0.20ex {  |}}\, }}

\def\Rbf{{\bf R}}

\def\be{\begin{equation}}
\def\ee{\end{equation}}

\def\scalefig#1{\epsfxsize #1\textwidth}
\newcommand{\n}{\mbox{${\bf n}(t)$}}

\def\Rxx{\Rbf_{\ssstyle X\kern-.1em X}}

\let\ssstyle=\scriptscriptstyle

\def\etal{{\it et al. \/}}

\def\ie{{\it i.e.,\ \/}}
\def\Kout{\setbox1=\hbox{\Huge\bf K}\hbox to
1.05\wd1{\hspace{.05\wd1}
}}
\def\Sout{\setbox1=\hbox{\Huge\bf S}\hbox to 1.05\wd1{\hspace{.05\wd1}
}}